\documentclass[mlabstract]{jmlr}

\usepackage{bm}
\usepackage{bbm}
\usepackage{float}
\usepackage{makecell}
\usepackage{adjustbox}
\usepackage{multirow}
\usepackage{wrapfig}
\usepackage{rotating}
\usepackage{enumitem}
\usepackage{footmisc}

\usepackage{longtable}% for long tables

\usepackage{booktabs}
\usepackage[load-configurations=version-1]{siunitx} % newer version
\theorembodyfont{\upshape}
\theoremheaderfont{\scshape}
\theorempostheader{:}
\theoremsep{\newline}

\jmlrvolume{}
\firstpageno{1}
\editors{Sophia Sanborn, Christian Shewmake, Simone Azeglio, Nina Miolane}

\jmlryear{2023}
\jmlrworkshop{NeurIPS Workshop on Symmetry and Geometry in Neural Representations}

\title[Symmetrization with Orbit Distance Minimization]{Learning Symmetrization for Equivariance with\\Orbit Distance Minimization}

  \author{\Name{Tien Dat Nguyen}\thanks{These authors contributed equally.} \Email{tiendat@kaist.ac.kr}\\
 \addr School of Computing, KAIST
 \AND
 \Name{Jinwoo Kim}\footnotemark[1] \Email{jinwoo-kim@kaist.ac.kr}\\
 \addr School of Computing, KAIST
 \AND
  \Name{Hongseok Yang} \Email{hongseok.yang@kaist.ac.kr}\\
 \addr School of Computing, KAIST
 \AND
  \Name{Seunghoon Hong} \Email{seunghoon.hong@kaist.ac.kr}\\
 \addr School of Computing, KAIST
 }

\begin{document}

\maketitle

\vspace{-0.6em}
\begin{abstract}
We present a general framework for symmetrizing an arbitrary neural-network architecture and making it equivariant with respect to a given group.
We build upon the proposals of \cite{kim2023learning, kaba2023equivariance} for symmetrization, and improve them by replacing their conversion of neural features into group representations, with an optimization whose loss intuitively measures the distance between group orbits.
This change makes our approach applicable to a broader range of matrix groups, such as the Lorentz group $\mathrm{O}(1,3)$, than these two proposals. %based on orbit separating invariants compared to \cite{kim2023learning, kaba2023equivariance}.
We experimentally show our method's competitiveness on the $\mathrm{SO}(2)$ image classification task, and also its increased generality on the task with $\mathrm{O}(1,3)$.
Our implementation will be made accessible at \url{https://github.com/tiendatnguyen-vision/Orbit-symmetrize}.
\end{abstract}

\vspace{-0.6em}
\section{Introduction}
\vspace{-0.1em}

Exploiting symmetries is a popular principle for developing an efficient learning system,  
which is typically realized by defining a hypothesis class of functions equivariant to a given group $G$ of symmetries. While a dominant approach to define such a hypothesis class has been to design a specific $G$ equivariant neural-network architecture~\citep{finzi2021a, villar2021scalars}, \emph{architecture-agnostic} alternatives are explored recently~\citep{puny2022frame, basu2023equituning, kaba2023equivariance, kim2023learning}. These alternatives are based on \emph{symmetrization}, where any unconstrained function $\phi_\theta:\mathcal{X}\to\mathcal{Y}$ is made $G$ equivariant by averaging it over transformations of inputs and outputs induced by certain group elements $g\in G$. In this work, we improve 
one of the most powerful symmetrization methods from \citep{kim2023learning, kaba2023equivariance}, which symmeterizes $\phi_\theta$ to the following function $\Phi_{\theta,\omega}$:
\begin{align}\label{eq:symmetrization}
    \Phi_{\theta,\omega}(\mathbf{x}) = \mathbb{E}_{\boldsymbol{\epsilon}}[g\cdot \phi_\theta(g^{-1}\cdot\mathbf{x})]\quad\mathrm{with}\quad\rho(g) = r(q_\omega(\mathbf{x},\boldsymbol{\epsilon})),
\end{align}
where $q_\omega:(\mathbf{x},\boldsymbol{\epsilon}) \mapsto \mathbf{h}\in\mathbb{R}^{n\times n}$ is a $G$~equivariant network, and $r: \mathbf{h} \mapsto \rho(g)$ is a $G$~equivariant \emph{contraction} operator producing the representation $\rho(g)$ of some element $g \in G$. 
% The contraction operator $r$ has an important role as it transforms neural feature $\mathbf{h}$ into a valid group representation $\rho(g)$ while preserving $G$ equivariance of the symmetrized function.

% In this paper, we investigate general principles to design the contractor operator.
A major issue with \equationref{eq:symmetrization} is that designing the contraction $r$ is often non-trivial; $r$ should produce a valid group representation $\rho(g)$ from an unstructured feature $\mathbf{h}$ while being $G$~equivariant itself.
Prior works employed hand-designed algorithms, such as Gram-Schmidt process for $\mathrm{O}(n)$~\citep{kaba2023equivariance}, but such algorithms are available only for certain groups~\citep{kim2023learning}.
%and a general design principle is not known.
% The contraction operator $r$ has an important role in these algorithms, as it transforms neural feature $\mathbf{h}$ into a valid group representation $\rho(g)$ while preserving $G$ equivariance of the symmetrized function.
% For example, for the orthogonal group $\mathrm{O}(n)$, Gram-Schmidt process serves as such operator as it converts any full-rank matrix to an orthogonal matrix $\rho(g)\in\mathrm{O}(n)$ while being $\mathrm{O}(n)$ equivariant  itself~\citep{kaba2023equivariance}.
% However, a major challenge in these methods is that, outside some groups such as $\mathrm{O}(n)$ and $\mathrm{S}_n$~\citep{kim2023learning} where hand-design is possible, finding the operator $r$ for a given group is itself a challenging problem and no general solution is known.
Our goal is to overcome this bottleneck and making the symmetrization work for broader group symmetries, such as the Lorentz group $\mathrm{O}(1, 3)$.

Our idea is to design a differentiable objective on~$q_\omega$ of which gradient-based optimization makes $q_\omega(\mathbf{x}, \boldsymbol{\epsilon})$ directly output valid group representations $\mathbf{h}\approx\rho(g)$, thereby removing the need for contraction $r$.
We design the objective in a principled manner as distance minimization on group orbit space.
Compared to symmetrization algorithms of \cite{kim2023learning, kaba2023equivariance}, this makes our approach applicable to a much broader range of matrix groups where orbit separating invariants are available.
We implement our method for the special orthogonal group $\mathrm{SO}(2)$ and the Lorentz group $\mathrm{O}(1, 3)$, and find that our objective can replace the known contraction $r$ for $\mathrm{SO}(2)$ with a negligible performance drop, and successfully achieves symmetrization based equivariance on the Lorentz group $\mathrm{O}(1, 3)$.

\section{Orbit Distance Minimization} \label{sec:orbit_distance_minimization}

\paragraph{Problem Definition}
Let $\rho:G\to\mathrm{GL}(n)$ be a group representation that associates each group element $g\in G$ an invertible matrix $\rho(g)\in\mathbb{R}^{n\times n}$.
For our $G$ equivariant neural network $q_\omega:(\mathbf{x},\boldsymbol{\epsilon})\mapsto\mathbf{h}\in\mathbb{R}^{n\times n}$, the group $G$ acts on the output space through the matrix multiplication of the representation $\mathbf{h}\mapsto g\cdot\mathbf{h}=\rho(g)\mathbf{h}$.
Our goal is to train $q_\omega$ such that its output is always a valid group representation $\mathbf{h}\in\rho(G)$, where $\rho(G)$ denotes the image of $\rho$.

\paragraph{Orbit Distance Minimization}
We will now present a training objective that contracts the output space of $q_\omega$ to valid group representations $\rho(G)\subset\mathbb{R}^{n\times n}$.
Our key idea is that, instead of working on $\mathbb{R}^{n\times n}$ directly, working on the orbit space (quotient) $\mathbb{R}^{n\times n}/G$ greatly simplifies the problem.
Let us write the orbit of an element $\mathbf{h}\in\mathbb{R}^{n\times n}$ under the action of $G$ as $[\mathbf{h}] = \{g\cdot\mathbf{h}:g\in G\}$.
The orbit space $\mathbb{R}^{n\times n}/G$ is defined accordingly as $\{[\mathbf{h}]:\mathbf{h}\in\mathbb{R}^{n\times n}\}$.
% The orbit space $\mathbb{R}^{n\times n}/G$ is the disjoint partition of $\mathbb{R}^{n\times n}$ formed by the set of all orbits of elements in $\mathbb{R}^{n\times n}$ under the action of $G$.
% A fundamental property of orbit space $\mathbb{R}^{n\times n}/G$ is that it forms a disjoint partition of $\mathbb{R}^{n\times n}$.
% This is because the action of $G$ induces an equivalence relation $\sim$ on $\mathbb{R}^{n\times n}$ that sets $\mathbf{h}\sim\mathbf{h}'$ if and only if $\mathbf{h}\in G\cdot \mathbf{h}'$, under which the orbits $G\cdot\mathbf{h}$ are the equivalence classes of $\mathbf{h}\in\mathbb{R}^{n\times n}$.

We now provide an observation that all valid group representations $\rho(G)$ precisely map onto a single point in the orbit space, which is the orbit of the identity matrix $\mathbf{I}\in\mathbb{R}^{n\times n}$ since we have $\rho(G) = \{\rho(g):g\in G\} = \{g\cdot\mathbf{I}:g\in G\} = [\mathbf{I}]$.
This implies, on the orbit space, our objective is understood as contracting all orbits of neural network outputs~$[\mathbf{h}]$ towards a fixed point target~$[\mathbf{I}]$.
Thus, if we can endow the orbit space with a distance metric $d:\mathbb{R}^{n\times n}/G\times \mathbb{R}^{n\times n}/G\to\mathbb{R}^+$, we can frame our training objective as distance minimization:
\begin{align}\label{eq:orbit_distance_minimization}
    w^* = \mathrm{arg}\,\underset{\omega}{\mathrm{min}}\ d([\mathbf{h}], [\mathbf{I}]),
\end{align}
where we remark that $q_\omega:(\mathbf{x},\boldsymbol{\epsilon})\mapsto\mathbf{h}\in\mathbb{R}^{n\times n}$ is our $G$ equivariant neural network.
We now show that the objective indeed contracts the output of $q_\omega$ exactly to $\rho(G)$ (proof in \ref{sec:main_proofs}):
\begin{theorem}\label{thm:optimality}
    The training objective in \equationref{eq:orbit_distance_minimization} achieves the global minimum with the value of $0$ if and only if $q_\omega$ always outputs valid group representations $\mathbf{h}\in\rho(G)$.
\end{theorem}
Our problem now reduces to defining a proper distance metric $d$ on the orbit space $\mathbb{R}^{n\times n}/G$.
The closest concept we could find in literature is the quotient metric~\citep{burago2001course}, but it is intractable as it involves infimum over an infinite set.
Instead, we propose a simple distance metric based on a class of functions called \emph{orbit separating invariants}: $G$ invariant functions $f$ that separate orbits $f(\mathbf{h}) \neq f(\mathbf{h}')\Longleftrightarrow[\mathbf{h}] \neq [\mathbf{h}']$~\citep{dym2022low}.
In detail, our distance metric can be defined as vector distance on outputs of $f$ (proof in \ref{sec:main_proofs}):
\begin{theorem}\label{thm:orbit_distance}
    Let $f:\mathbb{R}^{n\times n}\to\mathbb{R}^k$ be an orbit separating invariant and $\|\cdot\|$ be vector norm.
    Then, $d([\mathbf{h}], [\mathbf{h}']) = \|f(\mathbf{h}) - f(\mathbf{h'})\|$ is a distance metric on the orbit space $\mathbb{R}^{n\times n}/G$.
\end{theorem}
With \theoremref{thm:orbit_distance}, we can use the below objective for optimization problem in \equationref{eq:orbit_distance_minimization}\footnote{For some groups, orbit separation of $f$ is guaranteed only for full-rank inputs~\citep{dym2022low}.
In this case, the optimality condition in \theoremref{thm:optimality} holds if we assume $\mathbf{h}$ to be full-rank.\label{note:orbit_separation}}:
\begin{align}\label{eq:orbit_distance_loss}
    w^* = \mathrm{arg}\,\underset{\omega}{\mathrm{min}}\ \|f(\mathbf{h}) - f(\mathbf{I})\|.
\end{align}
In practice, it is desirable to have a differentiable objective such that we can perform a gradient-based optimization.
Since $q_\omega:(\mathbf{x},\boldsymbol{\epsilon})\mapsto\mathbf{h}$ is already a neural network, the objective in \equationref{eq:orbit_distance_loss} would be differentiable almost everywhere with respect to $\omega$ if we choose $f$~and $\|\cdot\|$ to be differentiable almost everywhere.

\paragraph{Discussion on Generality}\label{sec:generality}
% As our work is motivated by the limitations of the contraction operator $r:\mathbf{h}\mapsto\rho(g)$, it is natural to ask how orbit separating invariants $f:\mathbb{R}^{n\times n}\to\mathbb{R}^k$ in \equationref{eq:orbit_distance_loss} do better.
We now discuss the results from invariant theory implying orbit separating invariants $f:\mathbb{R}^{n\times n}\to\mathbb{R}^k$ of bounded dimension $k\leq 2n^2+1$ exist for a very general class of matrix groups and they are differentiable everywhere in general.

The existence of orbit separating invariants has been mainly shown for \emph{linearly reductive groups} including $\mathrm{GL}(n)$, semi-simple groups $\mathrm{SL}(n)$, $\mathrm{O}(n)$, $\mathrm{SO}(n)$, finite group $\mathrm{S}_n$, and also $\mathrm{O}(s, n-s)$.
Most of the results are derived from the concept of \emph{invariant polynomials}, which are polynomials on matrix entries that are $G$ invariant.
To elaborate, consider a group $G$ acting on $\mathbb{R}^{n\times n}$, and let $\mathcal{S}$ be the set of all invariant polynomials $f':\mathbb{R}^{n\times n}\to\mathbb{R}$.
We call a set of invariant polynomials $\{f_1, ..., f_k\}$ the \emph{generating set} if every $f'\in\mathcal{S}$ can be written as $f'(\cdot) = h(f_1(\cdot), ..., f_k(\cdot))$ using some polynomial $h:\mathbb{R}^k\to\mathbb{R}$.
For every linearly reductive group, Weyl's theorem~\citep{weyl1946the} guarantees the existence of a finite generating set.
Furthermore, for many subclasses of these groups, it has been shown that this set separates orbits in $\mathbb{R}^{n\times n}$ whose closures do not intersect\footnote{For compact groups this guarantees orbit separation; for closed non-compact groups this guarantees orbit separation for full-rank inputs~\citep{dym2022low} which relates to footnote~\ref{note:orbit_separation}.}~\citep{dym2022low, derksen2015computational}, allowing us to use their stack $f:\mathbb{R}^{n\times n}\to\mathbb{R}^k$ as our orbit separating invariant.
Such $f$ is differentiable everywhere as it is a stack of polynomials.
For many groups, the generating set is known from the invariant theory, and so is $f$ (see \ref{sec:supplementary_orbit_separating_invariants}).

Now consider the dimension $k$ of the separating invariant $f:\mathbb{R}^{n\times n}\to\mathbb{R}^k$ which is the size of the generating set $\{f_1, ..., f_k\}$ in our context.
While this can be large for some groups, \cite{dym2022low} has shown that random linear projection can almost always reduce it to a set of $2n^2+1$ polynomials that still separates orbits (see \ref{sec:projection}), which also reduces $f$.

\paragraph{Final Model}
Our original goal is to symmetrize a function $\phi_\theta:\mathcal{X}\to\mathcal{Y}$ to be $G$ equivariant.
We define our symmetrization as follows by removing contraction $r$ from \equationref{eq:symmetrization}:
\begin{align}\label{eq:final_model}
    \Phi_{\theta,\omega}(\mathbf{x}) = \mathbb{E}_{\mathbf{h}}[\mathbf{h}\cdot\phi_\theta(\mathbf{h}^{-1}\cdot\mathbf{x})]\quad\mathrm{where}\quad \mathbf{h} = q_\omega(\mathbf{x}, \boldsymbol{\epsilon}).
\end{align}
Given training pairs $(\mathbf{x}, \mathbf{y})$ of input $\mathbf{x}\in\mathcal{X}$ and label $\mathbf{y}\in\mathcal{Y}$, we train for the joint objective of task loss $\mathcal{L}$ and the orbit distance loss (\equationref{eq:orbit_distance_loss}) weighted by a hyperparameter $\lambda$:
\begin{align}\label{eq:training}
    \theta^*, \omega^* = \mathrm{arg}\,\underset{\theta, \omega}{\mathrm{min}}\ \mathcal{L}(\mathbf{y}, \Phi_{\theta,\omega}(\mathbf{x})) + \lambda\,\mathbb{E}_{\mathbf{h}}\|f(\mathbf{h}) - f(\mathbf{I})\|.
\end{align}
Intuitively, if the orbit loss is $\approx0$, we would have $\mathbf{h}\approx\rho(g)$ and the model would closely achieve the symmetrization in \equationref{eq:symmetrization} while not requiring contraction $r:\mathbf{h}\mapsto\rho(g)$.
A formal theoretical analysis on $G$ equivariance and universality of $\Phi_{\theta,\omega}$ is provided in \ref{sec:proof_equivariance_universality}.

\vspace{-0.8em}
\section{Experiments}
\vspace{-0.2em}
We evaluate our approach on two selected matrix groups: the special orthogonal group in two dimensions $\mathrm{SO}(2)$ and the Lorentz group $\mathrm{O}(1, 3)$.
% For the $\mathrm{SO}(2)$ group, an equivariant contraction operator $r:\mathbf{h}\mapsto\rho(g)$ has been designed by \cite{kim2023learning}, and we test whether orbit distance minimization can replace it.
% For the Lorentz group $\mathrm{O}(1, 3)$, the contraction operator is not known, so our experiments offer the first results on symmetrization based equivariance to our knowledge.
Experimental details are in \ref{sec:experimental_details}.

\paragraph{Image Classification} For $\mathrm{SO}(2)$, we use the Rotated MNIST~\citep{larochelle2007an}, a common benchmark for equivariant models~\citep{cohen2016group, finzi2020generalizing} with randomly rotated 62,000 digits ($\mathrm{SO}(2)$ invariance).
We follow the setup of \cite{kaba2023equivariance} and use the same 7-layer CNN as our base function $\phi_\theta$.
For the symmetrizer $q_\omega$, we use 2-layer EMLP~\citep{finzi2021a} of 64 hidden dimensions.
For training (\equationref{eq:training}), we use cross entropy for task loss~$\mathcal{L}$; for orbit distance loss $\lambda\,\|f(\mathbf{h})-f(\mathbf{I})\|$ we use the orbit separating invariant $f(\mathbf{h}) = [\mathrm{vec}(\mathbf{h}^\top\mathbf{h}), \mathrm{det}(\mathbf{h})]$~\citep{dym2022low}, L1 norm, $\lambda=1$.

\begin{table}[!t]
\floatconts
 {tab:experiments}
 {\vspace{-0.7em}\caption{Experimental results on $\mathrm{SO}(2)$ and $\mathrm{O}(1, 3)$ group symmetries.}\vspace{-0.7em}}
 {
   \subtable[Rotated MNIST, $\mathrm{SO}(2)$.]{
    \label{tab:rotated_mnist}
    \begin{adjustbox}{height=0.13\textwidth}
        \begin{tabular}{lc}
        \toprule
        Method & Test Error $\%\downarrow$ \\
        \midrule
        GCNN (p4) & $2.36 \pm 0.15$ \\
        GCNN (p64) & $2.28 \pm 0.10$ \\
        \midrule
        CNN & $4.90 \pm 0.20$ \\
        CNN-Aug. & $3.30 \pm 0.20$ \\
        \midrule
        CNN-Canonical. & $2.32 \pm 0.18$ \\
        CNN-PS & $2.21 \pm 0.28$ \\
        \midrule
        CNN-Canonical.-Orbit \textbf{(Ours)} & $2.44 \pm 0.12$ \\
        CNN-PS-Orbit \textbf{(Ours)} & $2.37 \pm 0.35$ \\
        \bottomrule
        \end{tabular}
    \end{adjustbox}
   }\qquad
   \subtable[Particle Scattering, $\mathrm{O}(1, 3)$.]{
    \label{tab:particle_scattering}
    \begin{adjustbox}{height=0.13\textwidth}
        \begin{tabular}{lc}
        \toprule
        Method & Test MSE $\downarrow$ \\
        \midrule
        Scalar MLP & $0.00171 \pm 0.00004$ \\
        \midrule
        MLP & $0.65381 \pm 0.23663$ \\
        MLP-Aug. & $0.09101 \pm 0.03107$ \\
        \midrule
        MLP-Canonical. & N/A \\
        MLP-PS & N/A \\
        \midrule
        MLP-Canonical.-Orbit \textbf{(Ours)} & $0.01027 \pm 0.00082$ \\
        MLP-PS-Orbit \textbf{(Ours)} & $0.00887 \pm 0.00070$ \\
        \bottomrule
        \end{tabular}
    \end{adjustbox}
   }
 }
\vspace{-1em}
\end{table}

We use the following baselines: CNN, CNN with $\mathrm{SO}(2)$ data augmentation, equivariant GCNN \citep{cohen2016group}, and CNN made $\mathrm{SO}(2)$ equivariant with symmetrization methods Canonicalization~\citep{kaba2023equivariance} and Probabilistic Symmetrization (PS)~\citep{kim2023learning} that follow \equationref{eq:symmetrization} except Canonicalization drops noise~$\boldsymbol{\epsilon}$.
We take the performances of CNN, data augmentation, and GCNN from \cite{kaba2023equivariance}, and train symmetrization baselines using $\mathrm{SO}(2)$ contraction of \cite{kim2023learning}.
The results are in \tableref{tab:rotated_mnist}.
Symmetrization improves CNN overall, as all symmetrized CNNs outperform data augmentation and perform on par with GCNN.
Within symmetrization, replacing contraction $r:\mathbf{h}\mapsto\rho(g)$ with our orbit distance loss leads to a negligible drop in performance.
This indicates orbit distance minimization can replace the role of contraction operator, with a slight tradeoff as $q_\omega$ takes an extra role of producing valid group representation $\mathbf{h}\approx\rho(g)$.

\vspace{-0.2em}
\paragraph{Particle Scattering}
For the Lorentz group $\mathrm{O}(1, 3)$, we use Particle Scattering synthetic regression dataset from \cite{finzi2021a} for matrix element in electron muon scattering ($\mathrm{O}(1, 3)$ invariance).
We use 10,000 train data, and use 1,000 validation and 1,000 test data that are randomly $\mathrm{O}(1, 3)$ transformed.
We use 3-layer MLP of 128 hidden dimensions as our base function $\phi_\theta$, and our symmetrizer $q_\omega$ is based on 3-layer Scalar MLP~\citep{villar2021scalars} of 128 hidden dimensions.
For training (\equationref{eq:training}), we use mean squared error for task loss~$\mathcal{L}$; for orbit distance loss $\lambda\,\|f(\mathbf{h})-f(\mathbf{I})\|$ we use the orbit separating invariant $f(\mathbf{h}) = [\mathrm{vec}(\mathbf{h}^\top\boldsymbol{\Lambda}\mathbf{h})]$, $\boldsymbol{\Lambda}=\mathrm{diag}([+1, -1, -1, -1])$~\citep{dym2022low}, L1 norm, $\lambda=1$.

We run the following baselines: MLP, MLP with $\mathrm{O}(1, 3)$ data augmentation, and invariant Scalar MLP~\citep{villar2021scalars}, with 3 layers of 128 hidden dimensions.
Symmetrization baselines in \equationref{eq:symmetrization} cannot be built as $\mathrm{O}(1, 3)$ contraction is not known.
The results are in \tableref{tab:particle_scattering}.
$\mathrm{O}(1, 3)$ symmetrized MLPs based on our method significantly outperform MLP as well as data augmentation.
To our knowledge, this is the first successful result in symmetrization based equivariance for $\mathrm{O}(1, 3)$, implying symmetrization can be applied to groups where contraction $r:\mathbf{h}\mapsto\rho(g)$ is not available.
Yet, our method has performance gap from invariant Scalar MLP.
An explanation is that Scalar MLP gets Minkowsky inner product $\mathbf{x}^\top\boldsymbol{\Lambda}\mathbf{x}$ as input which is heavily correlated to label function of particle scattering, but our base MLP gets transformed input $\approx\rho(g)^{-1}\cdot\mathbf{x}$ which requires additional processing.
We leave closing this gap as an important future research direction.

\vspace{-0.2em}
\paragraph{Conclusion}
We proposed orbit distance minimization, a framework for symmetrization based equivariant learning.
Our method is competitive on $\mathrm{SO}(2)$ invariant classification and successfully achieves symmetrization based equivariance on Lorentz group $\mathrm{O}(1, 3)$.

\paragraph{Acknowledgements}
This work was supported in part by the National Research Foundation of Korea (NRF) grant (NRF2021R1C1C1012540) and Institute of Information and communications Technology Planning and evaluation (IITP) grant (2021-0-00537 and 2021-0-02068) funded by the Korea government (MSIT).

\bibliography{main}

\begin{thebibliography}{18}
\providecommand{\natexlab}[1]{#1}
\providecommand{\url}[1]{\texttt{#1}}
\expandafter\ifx\csname urlstyle\endcsname\relax
  \providecommand{\doi}[1]{doi: #1}\else
  \providecommand{\doi}{doi: \begingroup \urlstyle{rm}\Url}\fi

\bibitem[Basu et~al.(2023)Basu, Sattigeri, Ramamurthy, Chenthamarakshan,
  Varshney, Varshney, and Das]{basu2023equituning}
Sourya Basu, Prasanna Sattigeri, Karthikeyan~Natesan Ramamurthy, Vijil
  Chenthamarakshan, Kush~R. Varshney, Lav~R. Varshney, and Payel Das.
\newblock Equi-tuning: Group equivariant fine-tuning of pretrained models.
\newblock In \emph{AAAI}, 2023.

\bibitem[Burago et~al.(2001)Burago, Burago, and Ivanov]{burago2001course}
D.~Burago, I.U.D. Burago, and S.~Ivanov.
\newblock \emph{A Course in Metric Geometry}.
\newblock American Mathematical Society, 2001.

\bibitem[Cohen and Welling(2016)]{cohen2016group}
Taco Cohen and Max Welling.
\newblock Group equivariant convolutional networks.
\newblock In \emph{ICML}, 2016.

\bibitem[Derksen and Kemper(2015)]{derksen2015computational}
Harm Derksen and Gregor Kemper.
\newblock \emph{Computational Invariant Theory}.
\newblock Springer, 2015.

\bibitem[Dym and Gortler(2022)]{dym2022low}
Nadav Dym and Steven~J. Gortler.
\newblock Low dimensional invariant embeddings for universal geometric
  learning.
\newblock \emph{arXiv}, 2022.

\bibitem[Dym and Maron(2020)]{dym2020universality}
Nadav Dym and Haggai Maron.
\newblock On the universality of rotation equivariant point cloud networks.
\newblock \emph{arXiv preprint arXiv:2010.02449}, 2020.

\bibitem[Finzi et~al.(2020)Finzi, Stanton, Izmailov, and
  Wilson]{finzi2020generalizing}
Marc Finzi, Samuel Stanton, Pavel Izmailov, and Andrew~Gordon Wilson.
\newblock Generalizing convolutional neural networks for equivariance to lie
  groups on arbitrary continuous data.
\newblock In \emph{ICML}, 2020.

\bibitem[Finzi et~al.(2021{\natexlab{a}})Finzi, Welling, and
  Wilson]{finzi2021a}
Marc Finzi, Max Welling, and Andrew~Gordon Wilson.
\newblock A practical method for constructing equivariant multilayer
  perceptrons for arbitrary matrix groups.
\newblock In \emph{ICML}, 2021{\natexlab{a}}.

\bibitem[Finzi et~al.(2021{\natexlab{b}})Finzi, Welling, and
  Wilson]{finzi2021practical}
Marc Finzi, Max Welling, and Andrew~Gordon Wilson.
\newblock A practical method for constructing equivariant multilayer
  perceptrons for arbitrary matrix groups.
\newblock In \emph{International conference on machine learning}, pages
  3318--3328. PMLR, 2021{\natexlab{b}}.

\bibitem[Hendrycks and Gimpel(2016)]{hendrycks2016gaussian}
Dan Hendrycks and Kevin Gimpel.
\newblock Gaussian error linear units (gelus).
\newblock \emph{arXiv}, 2016.

\bibitem[Kaba et~al.(2023)Kaba, Mondal, Zhang, Bengio, and
  Ravanbakhsh]{kaba2023equivariance}
S{\'{e}}kou{-}Oumar Kaba, Arnab~Kumar Mondal, Yan Zhang, Yoshua Bengio, and
  Siamak Ravanbakhsh.
\newblock Equivariance with learned canonicalization functions.
\newblock In \emph{ICML}, 2023.

\bibitem[Kim et~al.(2023)Kim, Nguyen, Suleymanzade, An, and
  Hong]{kim2023learning}
Jinwoo Kim, Tien~Dat Nguyen, Ayhan Suleymanzade, Hyeokjun An, and Seunghoon
  Hong.
\newblock Learning probabilistic symmetrization for architecture agnostic
  equivariance.
\newblock \emph{arXiv}, 2023.

\bibitem[Larochelle et~al.(2007)Larochelle, Erhan, Courville, Bergstra, and
  Bengio]{larochelle2007an}
Hugo Larochelle, Dumitru Erhan, Aaron~C. Courville, James Bergstra, and Yoshua
  Bengio.
\newblock An empirical evaluation of deep architectures on problems with many
  factors of variation.
\newblock In \emph{ICML}, 2007.

\bibitem[Puny et~al.(2022)Puny, Atzmon, Smith, Misra, Grover, Ben{-}Hamu, and
  Lipman]{puny2022frame}
Omri Puny, Matan Atzmon, Edward~J. Smith, Ishan Misra, Aditya Grover, Heli
  Ben{-}Hamu, and Yaron Lipman.
\newblock Frame averaging for invariant and equivariant network design.
\newblock In \emph{ICLR}, 2022.

\bibitem[Riba et~al.(2020)Riba, Mishkin, Ponsa, Rublee, and
  Bradski]{riba2020kornia}
Edgar Riba, Dmytro Mishkin, Daniel Ponsa, Ethan Rublee, and Gary Bradski.
\newblock Kornia: an open source differentiable computer vision library for
  pytorch.
\newblock In \emph{Proceedings of the IEEE/CVF Winter Conference on
  Applications of Computer Vision}, pages 3674--3683, 2020.

\bibitem[Segol and Lipman(2019)]{segol2019universal}
Nimrod Segol and Yaron Lipman.
\newblock On universal equivariant set networks.
\newblock \emph{arXiv preprint arXiv:1910.02421}, 2019.

\bibitem[Villar et~al.(2021)Villar, Hogg, Storey{-}Fisher, Yao, and
  Blum{-}Smith]{villar2021scalars}
Soledad Villar, David~W. Hogg, Kate Storey{-}Fisher, Weichi Yao, and Ben
  Blum{-}Smith.
\newblock Scalars are universal: Equivariant machine learning, structured like
  classical physics.
\newblock In \emph{NeurIPS}, 2021.

\bibitem[Weyl(1946)]{weyl1946the}
Hermann Weyl.
\newblock \emph{The Classical Groups: Their Invariants and Representations}.
\newblock Princeton University Press, 1946.

\end{thebibliography}

\newpage

\appendix

\section{Proofs and Supplementary Discussions}
\subsection{Proof of Theorem~\ref{thm:optimality} and Theorem~\ref{thm:orbit_distance}}\label{sec:main_proofs}
Before the proofs, we provide a formal definition of distance metric on the orbit space.
\begin{definition}\label{defn:distance_metric}
    A function $d:\mathbb{R}^{n\times n}/G\times \mathbb{R}^{n\times n}/G\to\mathbb{R}^+$ is a distance metric on the orbit space $\mathbb{R}^{n\times n}/G$ if it satisfies the following conditions for all orbits $[\mathbf{h}], [\mathbf{h}'], [\mathbf{h}'']\in \mathbb{R}^{n\times n}/G$:
    \begin{enumerate}
        \item $d([\mathbf{h}], [\mathbf{h}']) \geq 0$ (non-negativity),
        \item $d([\mathbf{h}], [\mathbf{h}']) = 0 \Longleftrightarrow [\mathbf{h}] = [\mathbf{h}']$ (identity of indiscernibles),
        \item $d([\mathbf{h}], [\mathbf{h}']) = d([\mathbf{h}'], [\mathbf{h}])$ (symmetry),
        \item $d([\mathbf{h}], [\mathbf{h}']) \leq d([\mathbf{h}], [\mathbf{h}'']) + d([\mathbf{h}''], [\mathbf{h}'])$ (triangle inequality).
    \end{enumerate}
\end{definition}
We now provide the proofs.
\begingroup
\def\thetheorem{\ref{thm:optimality}}
\begin{theorem}
    The training objective in \equationref{eq:orbit_distance_minimization} achieves the global minimum with the value of $0$ if and only if $q_\omega$ always outputs valid group representations $\mathbf{h}\in\rho(G)$.
\end{theorem}
\addtocounter{theorem}{-1}
\endgroup
\begin{proof}
Recall the definition of an orbit $[\mathbf{h}] = \{g\cdot\mathbf{h}:g\in G\}$.
($\Longrightarrow$) If $d([\mathbf{h}], [\mathbf{I}])=0$, since $d$ is a distance metric, we have $[\mathbf{h}]=[\mathbf{I}]$ from identity of indiscernibles.
This implies $[\mathbf{h}]=\rho(G)$ since $[\mathbf{I}] = \{g\cdot\mathbf{I}:g\in G\} = \{\rho(g):g\in G\} = \rho(G)$.
On the orbit $[\mathbf{h}]=\{g\cdot\mathbf{h}:g\in G\}$, by selecting the identity element $\mathrm{id}\in G$ we get $\mathbf{h}\in[\mathbf{h}]$, thus $\mathbf{h}\in\rho(G)$.
($\Longleftarrow$) 
If $q_\omega$ always outputs valid group representations $\mathbf{h}\in\rho(G)$, we can write $\mathbf{h}=\rho(h)$ for some $h\in G$.
Then we have $[\mathbf{h}]=\{g\cdot\mathbf{h}:g\in G\} = \{g\cdot\rho(h):g\in G\} = \{(gh)\cdot\mathbf{I}:g\in G\} = \{g'\cdot\mathbf{I}:g'h^{-1}\in G\} = \{g'\cdot\mathbf{I}:g'\in G\} = [\mathbf{I}]$.
Note that we used the associativity and invertibility of group elements as well as the fact that right operation by $h\in G$ maps a group $G$ to itself.
Since $[\mathbf{h}]=[\mathbf{I}]$ and $d$ is a distance metric, we have that $d([\mathbf{h}], [\mathbf{I}])=0$ due to identity of indiscernibles, which is a global minimum due to non-negativity.
\end{proof}
\begingroup
\def\thetheorem{\ref{thm:orbit_distance}}
\begin{theorem}
    Let $f:\mathbb{R}^{n\times n}\to\mathbb{R}^k$ be an orbit separating invariant and $\|\cdot\|$ be vector norm.
    Then, $d([\mathbf{h}], [\mathbf{h}']) = \|f(\mathbf{h}) - f(\mathbf{h}')\|$ is a distance metric on the orbit space $\mathbb{R}^{n\times n}/G$.
\end{theorem}
\addtocounter{theorem}{-1}
\endgroup
\begin{proof}
We first note that a vector norm $\|\cdot\|$ induces a distance metric $d'(\mathbf{u}, \mathbf{v})=\|\mathbf{u}-\mathbf{v}\|$, which is called the induced metric.
We now explicitly show that $d([\mathbf{h}], [\mathbf{h}']) = \|f(\mathbf{h}) - f(\mathbf{h}')\|$ satisfies the four conditions in \definitionref{defn:distance_metric}.
Since $\|\cdot\|$ is a vector norm, non-negativity clearly holds.
Since $f$ is an orbit separating invariant $[\mathbf{h}] \neq [\mathbf{h}'] \Longleftrightarrow f(\mathbf{h}) \neq f(\mathbf{h}')$, by invoking the identity of indiscernibles of the induced metric we have $[\mathbf{h}] \neq [\mathbf{h}'] \Longleftrightarrow \|f(\mathbf{h}) - f(\mathbf{h}')\| \neq 0$, which proves the identity of indiscernibles for $d([\mathbf{h}], [\mathbf{h}'])$.
Symmetry and triangle inequality of $d([\mathbf{h}], [\mathbf{h}'])$ are inherited from the symmetry and triangle inequality of the induced metric, as we have $d([\mathbf{h}], [\mathbf{h}']) = \|f(\mathbf{h}) - f(\mathbf{h}')\| = \|f(\mathbf{h}') - f(\mathbf{h})\| = d([\mathbf{h}'], [\mathbf{h}])$ for symmetry and $d([\mathbf{h}], [\mathbf{h}']) = \|f(\mathbf{h}) - f(\mathbf{h}')\| \leq \|f(\mathbf{h}) - f(\mathbf{h}'')\| + \|f(\mathbf{h}'') - f(\mathbf{h}')\| = d([\mathbf{h}], [\mathbf{h}'']) + d([\mathbf{h}''], [\mathbf{h}'])$ for triangle inequality.
Therefore, $d([\mathbf{h}], [\mathbf{h}'])$ is a distance metric on the orbit space.
\end{proof}

\subsection{Supplementary Discussion on Orbit Separating Invariants}\label{sec:supplementary_orbit_separating_invariants}
\begin{table}[!t]
\centering
\caption{Orbit separating invariants for some group actions from (\cite{dym2022low}), along with the domain on which orbit separation is guaranteed.}
\begin{adjustbox}{width=0.9\textwidth}
\begin{tabular}{cccc} 
\toprule
Group & Domain & Orbit separating invariant & Dimension \\
\midrule
$\mathrm{S}_n$ & $\mathbb{R}^{n \times n}$ & $[\phi_{\alpha}(\mathbf{h}) = \sum_{j=1}^{n}\mathbf{h}_j^{\alpha}], \quad\alpha \in \mathbb{Z}^d_{\geq 0}, |\alpha| \leq n$ & $\binom{2n}{n}$  \\
\midrule
$\mathrm{O}(n)$ &  $\mathbb{R}^{n \times n}$ & $[\mathrm{vec}(\mathbf{h}^\top\mathbf{h})]$ & $n^2$ \\
\midrule
$\mathrm{SO}(n)$ & $\mathbb{R}^{n \times n}$  & $[\mathrm{vec}(\mathbf{h}^\top\mathbf{h}), \det\mathbf{h}]$ & $n^2 + 1$ \\
\midrule
$\mathrm{O}(1, n-1)$ & $\mathbb{R}^{n \times n}_{\mathrm{full}}$ & $[\mathrm{vec}(\mathbf{h}^\top\boldsymbol{\Lambda}\mathbf{h})],\quad \boldsymbol{\Lambda}=\mathrm{diag}([+1, -1, ..., -1])$   & $n^2$ \\
\midrule
$\mathrm{SL}(n)$ & $\mathbb{R}^{n \times n}_{\mathrm{full}}$ & $[\det\mathbf{h}]$ & 1 \\
\midrule
$\mathrm{GL}(n)$ & $\mathbb{R}^{n \times n}_{\mathrm{full}}$ & $[\det^2 (\mathbf{h}\mathbf{W}_i)/\det^{-1}(\mathbf{h}\mathbf{h}^\top)],\quad \mathbf{W}_i \in \mathbb{R}^{n \times n}, i=1,...2n^2+1$ & $2n^2+1$ \\
\bottomrule
\end{tabular}
\end{adjustbox}
% \vspace{-1em}
\label{table:summary_groups}
\end{table}
In this section, we provide examples of known orbit separating invariants for certain linearly reductive groups in \tableref{table:summary_groups} along with supplementary discussion for $\mathrm{S}_n$ and $\mathrm{GL}(n)$ groups.
For the symmetric group $\mathrm{S}_n$, the orbit separating invariant is implemented upon a set of invariant functions known as multi-dimensional power sum polynomials~\citep{dym2020universality, segol2019universal}.
For a given input $\mathbf{h} \in \mathbb{R}^{n \times n}$, this invariant is written as:
\begin{align}
    \phi_{\alpha}(\mathbf{h}) = \sum_{j=1}^{n}\mathbf{h}_j^{\alpha},\quad\alpha \in \mathbb{Z}^d_{\geq 0}, |\alpha| \leq n,
\end{align}
where $\mathbf{h}_j$ is $j$th row of $\mathbf{h}$, $\alpha = (\alpha_1, ... \alpha_n) \in \mathbb{N}^n$ is a multi-index, and $\mathbf{h}_j^{\alpha} = \mathbf{h}_{j,\alpha_1}\times...\times\mathbf{h}_{j,\alpha_n}$.
% As discussed in (\cite{dym2022low, segol2019universal}), this set indeed separates orbits in $\mathbb{R}^{n \times n}$ with respect to the group $\mathrm{S}_n$.
For the general linear group $\mathrm{GL}(n)$, on the contrary to other groups in \tableref{table:summary_groups} whose orbit separating invariants are based on the generating set of invariant polynomials, the generating set with respect to action of $\mathrm{GL}(n)$ consists solely of constant polynomial, which cannot be used to implement the orbit separating invariant $f$.
Instead, \cite{dym2022low} has shown that $f$ can be built by adopting a family of rational invariants as follows:
\begin{align}
    q(\mathbf{h}, \mathbf{W}) = \frac{\det(\mathbf{h}\mathbf{W})^2}{\det(\mathbf{h}\mathbf{h}^\top)},\quad \mathbf{h}, \mathbf{W} \in \mathbb{R}^{n \times n},
\end{align}
where, \cite{dym2022low} shows that, for almost every $\mathbf{W}_1, ... \mathbf{W}_{2n^2+1}$ randomly sampled from $\mathbb{R}^{n \times n}$ the set of functions $f = \{ q(\mathbf{h}, \mathbf{W}_i):i=1,..2n^2+1 \}$ separates orbits of invertible matrices in $\mathbb{R}^{n \times n}$.
This example demonstrates that our method has the possibility to be applied to linearly reductive groups even when non-trivial generating set of invariant polynomial is not available, as long as an alternative orbit separating set is found.

% \vspace{-0.5em}
\subsection{Supplementary Discussion on Random Projection for Scalability}\label{sec:projection}
In this section, we discuss the method for controlling the dimension of separating invariants $f:\mathbb{R}^{n\times n}\to\mathbb{R}^k$ to be $k\leq 2n^2+1$ as we have discussed in \sectionref{sec:orbit_distance_minimization}.
Specifically, we summarize the projection technique suggested in \cite{dym2022low}, which produce a smaller set of orbit separating invariants from the generating set.
The idea is summarized below, which is an immediate consequence of corollary 1.9 in \cite{dym2022low}:
\begin{lemma}
    Consider a group $G$ that acts on $\mathbb{R}^{n \times n}$ and let $f: \mathbb{R}^{n \times n} \to \mathbb{R}^k$ be an orbit separating invariant for this action.
    For almost every vectors $\mathbf{w}^{(1)}, ... \mathbf{w}^{(2n^2+1)} \in \mathbb{R}^k$ sampled randomly, the function $\hat{f}: \mathbb{R}^{n \times n} \to \mathbb{R}^{2n^2+1}$ with component functions defined as:
    \begin{equation}\label{eq:projection}
        \hat{f}_j(\mathbf{x}) = \sum_{i=1}^{k} \mathbf{w}_i^{(j)}f_i(\mathbf{x}), \quad j=1,...2n^2+1,
    \end{equation}
    is also an orbit separating invariant.
\end{lemma}
With this technique, given an orbit separating invariant $f$ in arbitrary high dimension $k$, we can get a new orbit separating invariant $\hat{f}$ in dimension $2n^2+1$. Furthermore, this technique preserves differentiability as $\hat{f}_j$ is merely a linear combination of $f_i$.
If $f$ is composed of polynomials, as in the many cases of linearly reductive groups~\citep{dym2022low}, we can avoid $k$ intermediate variables $f_i(\mathbf{x})$ in \equationref{eq:projection} by fixing linear projections $\mathbf{w}$ and contracting $k$ polynomials $f_i$ into $2n^2+1$ polynomials $\hat{f}_j$ before computing $\hat{f}(\mathbf{x})$.

\subsection{Proof of Equivariance and Universality}\label{sec:proof_equivariance_universality}
In this section, we formally prove the $G$ equivariance and universality of our symmetrized model $\Phi_{\theta,\omega}$ in \equationref{eq:final_model}.
We define some notations and assumptions beforehand.

\paragraph{Definitions}
In the proofs, we set our group to be a matrix group $G \subset \mathrm{GL}(n)$, and define representations $\rho:G\to\mathrm{GL}(n)$ and $\rho^{gl}:\mathrm{GL}(n)\to\mathrm{GL}(n)$ under the restriction of being \emph{aligned} $\rho(g)=\rho^{gl}(g)$ for all $g\in G$ without the loss of generality.
For example, $\rho$ and $\rho^{gl}$ can be chosen as identity maps.
Recall our $G$ equivariant neural network $q_\omega:(\mathbf{x},\boldsymbol{\epsilon})\mapsto\mathbf{h}\in\mathbb{R}^{n\times n}$ in \equationref{eq:final_model}.
Given that $\boldsymbol{\epsilon}$ is a random variable, we can consider the implicit probabilistic distribution characterized by $q_\omega$, which we denote as $p_\omega(\mathbf{h}|\mathbf{x})$.
Further assuming that $\mathbf{h}$ is full-rank, we have $\mathbf{h}=g\in\mathrm{GL}(n)$, and we write our distribution as $p_\omega(g|\mathbf{x})$.
% \paragraph{Setup and notations}
% We fix the matrix group $G \subset \mathrm{GL}(n)$ that will be considered in all following proofs. We denote $\rho, \rho^{gl}$ as group representation for $G$ and $\mathrm{GL}(n)$ respectively such that restriction of $\rho^{gl}$ to $G$ aligns with $\rho$ (i.e $\rho^{gl}(g) = \rho(g), \forall g \in G$). For example, $\rho, \rho^{gl}$ can be chosen as trivial matrix  group representation to satisfy this restriction relationship. \\ 
% We denote $p_\omega(g|\mathbf{x})$ as the conditional distribution for all possible values $g \in \mathrm{GL}(n)$ of feature $\mathbf{h}$ in \equationref{eq:final_model}). This distribution is obviously characterized by the interface $q_\omega: \mathcal{X} \times \mathcal{E} \to \mathrm{GL}(n)$.\\
Based on the notations, we can rewrite \equationref{eq:final_model} more formally as follows: 
\begin{equation}\label{eq:symmetrization_reform}
    \Phi_{\theta,\omega}(\mathbf{x}) = \mathbb{E}_{p_\omega(g|\mathbf{x})}[\rho^{gl}_{\mathcal{Y}}(g) \phi_\theta(\rho^{gl}_{\mathcal{X}}(g)^{-1} \mathbf{x})],
\end{equation}
where $\rho^{gl}_{\mathcal{X}}:\mathrm{GL}(n)\to\mathrm{GL}(\mathcal{X})$ and $\rho^{gl}_{\mathcal{Y}}:\mathrm{GL}(n)\to\mathrm{GL}(\mathcal{Y})$ are the representations of $\mathrm{GL}(n)$ on the input space and output space of the base function $\phi_\theta:\mathcal{X}\to\mathcal{Y}$, respectively.

\subsubsection{Proof of Equivariance}
In this section, we prove that even if our $G$ equivariant parameterized distribution $p_\omega(g|\mathbf{x})$ in \equationref{eq:symmetrization_reform} over which the base function $\phi_\theta$ is symmetrized is not strictly supported on $G$, the entire symmetrized function $\Phi_{\theta,\omega}$ would still retain $G$ equivariance.

\begin{definition}\label{defn:equivariant_distribution}
    Consider a group $G\subset\mathrm{GL}(n)$ acting on a vector space $\mathcal{X}$.
    The conditional probabilistic distribution $p_\omega(g|\mathbf{x})$ for $g \in \mathrm{GL}(n)$ and $\mathbf{x} \in \mathcal{X}$ is $G$ equivariant if it satisfies:
    \begin{equation}
        p_\omega(g|\mathbf{x}) = p_\omega(g'g|\rho_{\mathcal{X}}(g')\mathbf{x}), \text{ } \forall g' \in G, g \in \mathrm{GL}(n), \mathbf{x} \in \mathcal{X},
    \end{equation}
    where $\rho_\mathcal{X}:G\to\mathrm{GL}(n)$ is the representation of $G$ on $\mathcal{X}$.
\end{definition}
This generalizes the notion of probabilistic $G$ equivariance in \cite{kim2023learning} to our framework, since the support of $p_\omega(\cdot|\mathbf{x})$ is not limited to $G$ but extended to $\mathrm{GL}(n)$.
We first prove that $G$ equivariance of $p_\omega$ is achieved with an appropriate choice of $q_\omega$ and $\boldsymbol{\epsilon}$:
\begin{lemma}\label{lemma:equivariant_distribution}
    If $q_\omega$ is $G$ equivariant and $p(\boldsymbol{\epsilon})$ is $G$ invariant under a representation $\rho_{\mathcal{E}}$ that satisfies $\mathrm{det}(\rho_{\mathcal{E}}(\boldsymbol{\epsilon}))=1\forall\boldsymbol{\epsilon}\in\mathcal{E}$, then the probabilistic distribution $p_\omega(g|\mathbf{x})$ characterized by $q_\omega$ is $G$ equivariant (\definitionref{defn:equivariant_distribution}).
\end{lemma}
\begin{proof}
    Our proof is inspired by the proof of Theorem 3 in \cite{kim2023learning}.
    Firstly, we interpret the probability $p_\omega(g|\mathbf{x}, \boldsymbol{\epsilon})$ as a delta distribution: 
    \begin{equation}
        p_\omega(g|\mathbf{x}, \boldsymbol{\epsilon}) = \delta(\rho^{gl}(g) = q_\omega(\mathbf{x}, \boldsymbol{\epsilon})).
    \end{equation}
    We marginalize over $p(\boldsymbol{\epsilon})$ to get $p_\omega(g|\mathbf{x})$:
    \begin{align}
        p_\omega(g | \mathbf{x}) &= \int_{\boldsymbol{\epsilon}}p_\omega(g|\mathbf{x}, \boldsymbol{\epsilon})p(\boldsymbol{\epsilon})d\boldsymbol{\epsilon}\nonumber \\
        &= \int_{\boldsymbol{\epsilon}} \delta(\rho^{gl}(g) = q_\omega(\mathbf{x}, \boldsymbol{\epsilon})) p(\boldsymbol{\epsilon})d\boldsymbol{\epsilon}.
    \end{align}
    Moreover, we have: 
    \begin{align}
        p_\omega(g' g | \rho_{\mathcal{X}}(g')\mathbf{x}) = \int_{\boldsymbol{\epsilon}} \delta(\rho^{gl}(g'g) = q_\omega(\rho_{\mathcal{X}}(g')\mathbf{x}, \boldsymbol{\epsilon})) p(\boldsymbol{\epsilon})d\boldsymbol{\epsilon}.
    \end{align}
    Since $\rho$ is a restriction of $\rho^{gl}$ into $G$, we automatically have $\rho^{gl}(g) = \rho(g)\ \forall g \in G$. Together with the equivariance of $q_\omega$, we have:
    \begin{align}
        q_\omega(\rho_{\mathcal{X}}(g')\mathbf{x}, \boldsymbol{\epsilon}) &= \rho(g')q_\omega(\mathbf{x}, \rho_{\mathcal{E}}(g')^{-1}\boldsymbol{\epsilon}) \nonumber \\
        &= \rho^{gl}(g')q_\omega(\mathbf{x}, \rho_{\mathcal{E}}(g')^{-1}\boldsymbol{\epsilon}).
    \end{align}
    This leads to:
    \begin{align}\label{eq:lem4_0}
        p_\omega(g' g | \rho_{\mathcal{X}}(g')\mathbf{x})  &= \int_{\boldsymbol{\epsilon}} \delta(\rho^{gl}(g'g) = \rho^{gl}(g')q_\omega(\mathbf{x}, \rho_{\mathcal{E}}(g')^{-1}\boldsymbol{\epsilon})) p(\boldsymbol{\epsilon})d\boldsymbol{\epsilon} \nonumber \\
        &= \int_{\boldsymbol{\epsilon}} \delta(\rho^{gl}(g) = q_\omega(\mathbf{x}, \rho_{\mathcal{E}}(g')^{-1}\boldsymbol{\epsilon})) p(\boldsymbol{\epsilon})d\boldsymbol{\epsilon}.
    \end{align}
    Next, to compute \equationref{eq:lem4_0}, we introduce a change of variable $\boldsymbol{\epsilon}' = \rho_{\mathcal{E}}(g')^{-1}\boldsymbol{\epsilon}$:
    \begin{align}
        p_\omega(g' g | \rho_{\mathcal{X}}(g')\mathbf{x}) = \int_{\boldsymbol{\epsilon}'} \delta(\rho^{gl}(g) = q_\omega(\mathbf{x}, \boldsymbol{\epsilon}')) p(\rho_{\mathcal{E}}(g') \boldsymbol{\epsilon}') \frac{1}{|\det \rho_{\mathcal{E}}(g')^{-1}|} d\boldsymbol{\epsilon}'.
    \end{align}
    Since $\rho_{\mathcal{E}}(\cdot)$ always has determinant $1$, we have $|\det \rho_{\mathcal{E}}(g')^{-1}| = 1$. Furthermore, the invariance of $p(\boldsymbol{\epsilon})$ with respect to $G$ gives $p(\rho_{\mathcal{E}}(g') \boldsymbol{\epsilon}') = p(\boldsymbol{\epsilon}')$. 
    Eventually, we get:
    \begin{align}
        p_\omega(g' g | \rho_{\mathcal{X}}(g')\mathbf{x}) &= \int_{\boldsymbol{\epsilon}'} \delta(\rho^{gl}(g) = q_\omega(\mathbf{x}, \boldsymbol{\epsilon}')) p(\boldsymbol{\epsilon}') d\boldsymbol{\epsilon}' \nonumber \\ 
        &= p_\omega(g | \mathbf{x}).
    \end{align}
    This finishes the proof.
\end{proof}
Next, we show the $G$ equivariance of the symmetrized model $\Phi_{\theta, \omega}$ (\equationref{eq:symmetrization_reform}):
\begin{theorem}\label{thm:equivariance}
    If $p_\omega$ is $G$ equivariant, then $\Phi_{\theta, \omega}$ is $G$ equivariant for arbitrary $\phi_\theta$.
\end{theorem}
\begin{proof}
We prove $\Phi_{\theta, \omega}(\rho_{\mathcal{X}}(g')\mathbf{x}) = \rho_{\mathcal{Y}}(g')\Phi_{\theta, \omega}(\mathbf{x})$ for all $\mathbf{x} \in \mathcal{X}$, $g' \in G$.
From \equationref{eq:symmetrization_reform}:
\begin{align}\label{eq:thm1_0}
    \Phi_{\theta, \omega}(\rho_{\mathcal{X}}(g')\mathbf{x}) = \mathbb{E}_{p_\omega(g|\rho_{\mathcal{X}}(g')\mathbf{x})} [\rho^{gl}_{\mathcal{Y}}(g) \phi_\theta(\rho^{gl}_{\mathcal{X}}(g)^{-1} \rho_{\mathcal{X}}(g')\mathbf{x})].
\end{align}
Let us define $h = g'^{-1}g \in \mathrm{GL}(n)$, then we have $g = g'h$. Because $p_{\omega}$ is $G$ equivariant, we have $p_\omega(g|\rho_{\mathcal{X}}(g')\mathbf{x}) = p_\omega(g'h|\rho_{\mathcal{X}}(g')\mathbf{x}) = p_\omega(h|\mathbf{x})$. Thus, \equationref{eq:thm1_0} becomes: 
\begin{align}
    \Phi_{\theta, \omega}(\rho_{\mathcal{X}}(g')\mathbf{x}) &= \mathbb{E}_{p_\omega(h|\mathbf{x})}[\rho^{gl}_{\mathcal{Y}}(g'h) \phi_\theta(\rho^{gl}_{\mathcal{X}}(g'h)^{-1} \rho^{gl}_{\mathcal{X}}(g')\mathbf{x} )] \nonumber \\
    &= \rho^{gl}_{\mathcal{Y}}(g') \mathbb{E}_{p_\omega(h|\mathbf{x})} [\rho^{gl}_{\mathcal{Y}}(h) \phi_\theta(\rho^{gl}_{\mathcal{X}}(h)^{-1} \mathbf{x})] \nonumber \\
    &= \rho^{gl}_{\mathcal{Y}}(g') \Phi_{\theta, \omega}(\mathbf{x}) \nonumber \\ 
    &= \rho_{\mathcal{Y}}(g') \Phi_{\theta, \omega}(\mathbf{x}).
\end{align}
The last equality is from the fact that $\rho$ is a restriction of $\rho^{gl}$ to $G$.
This finishes the proof.
\end{proof}

\subsubsection{Proof of Universality}

In this section, we prove that even if our $G$ equivariant parameterized distribution $p_\omega(g|\mathbf{x})$ in \equationref{eq:symmetrization_reform} is not strictly supported on $G$, the entire symmetrized function $\Phi_{\theta,\omega}$ is still a universal approximator of arbitrary $G$ equivariant functions as long as $p_\omega(g|\mathbf{x})$ can approximate a $G$~equivariant distribution $h(g|\mathbf{x})$ which is compactly supported on $G$.

\begin{definition}\label{defn:compact_contract}
    We say that a $G$ equivariant distribution $h(g|\mathbf{x})$ on $g\in \mathrm{GL}(n)$ and $\mathbf{x}\in\mathcal{X}$ is \emph{compactly supported on $G$} if (1) $h(g|\mathbf{x})$ is supported on $G$ for all $\mathbf{x}$, and (2) for any compact set $\mathcal{K} \subset \mathcal{X}$, the union of the support $\cup_{\mathbf{x}\in\mathcal{K}}\,\mathrm{supp}\,h(g|\mathbf{x})$ is compact.
    % $\{g|h(g|\mathbf{x}) > 0$ for some $\mathbf{x}\in \mathcal{K}\}$ and $\{g^{-1}|h(g|\mathbf{x})>0$ for some $\mathbf{x} \in \mathcal{K} \}$ are both compact.
\end{definition}

\begin{definition}\label{defn:approximate_compact_contract}
    We say that a parameterized $G$ equivariant distribution $p_\omega(g | \mathbf{x})$ on $g \in \mathrm{GL}(n)$ and $\mathbf{x} \in \mathcal{X}$ is \emph{approximately compactly supported on $G$} if there exists a distribution $h(g|\mathbf{x})$ compactly supported on $G$ such that, for any compact set $\mathcal{K} \subset \mathcal{X}$ and any $0 < \alpha < 1$, there exists a choice of parameters $\omega$ that (1) the following holds for all $g \in G$ and $\mathbf{x} \in \mathcal{K}$:
    \begin{align}\label{eq:p_approaching_h}
        p_{\omega}(g | \mathbf{x}) \geq (1 - \alpha)h(g|\mathbf{x}),
    \end{align}
    and (2) the union of the support $\cup_{\mathbf{x}\in\mathcal{K}}\,\mathrm{supp}\,p_\omega(g|\mathbf{x})$ is a subset of a compact set $\mathcal{H}$ that only depend on the given $h$ and $\mathcal{K}$.
    % \begin{equation}
    %     \begin{cases}
    %         p_{\omega}(g | \mathbf{x}) \geq (1 - \alpha)h(g|\mathbf{x}),\quad\forall g \in G, \mathbf{x} \in \mathcal{K} \\
    %         p_{\omega} (\cdot | \mathcal{K}) \subset \mathcal{H} \\
    %     \end{cases}
    % \end{equation}
\end{definition}

\begin{lemma}\label{lemma:universal_contract}
    Let $h(g|\mathbf{x})$ be a $G$ equivariant distribution compactly supported on $G$.
    For any function $\phi_\theta: \mathcal{X} \to \mathcal{Y}$, we define its symmetrization $\kappa_{\theta}: \mathcal{X} \to \mathcal{Y}$ over $h(g|\mathbf{x})$ as follows:
    % We build a function $\kappa_{\theta}: \mathcal{X} \to \mathcal{Y}$ by symmetrizing any base function $\phi_\theta: \mathcal{X} \to \mathcal{Y}$  over $h(g|\mathbf{x})$ as following: 
    \begin{equation}
        \kappa_{\theta}(\mathbf{x}) = \mathbb{E}_{h(g|\mathbf{x})}[\rho_{\mathcal{Y}}(g) \phi_\theta(\rho_{\mathcal{X}}(g)^{-1}\mathbf{x})].
    \end{equation}
    Then, $\kappa_{\theta}$ is $G$ equivariant.
    Furthermore, $\kappa_{\theta}$ is a universal approximator of $G$ equivariant functions if $\phi_\theta$ is a universal approximator.
\end{lemma}
\begin{proof}
    Our proof is inspired by the proof of theorem 2 of \cite{kim2023learning}.
    We first prove the $G$ equivariance of $\kappa_\theta$ by showing $\kappa_{\theta}(\rho_{\mathcal{X}}(g')\mathbf{x}) = \rho_{\mathcal{Y}}(g')\kappa_{\theta}(\mathbf{x})\forall g \in G, \mathbf{x} \in \mathcal{X}$.
    We write:
    \begin{align}\label{eq:fix_universal_00}
        \kappa_{\theta}(\rho_{\mathcal{X}}(g')\mathbf{x}) = \mathbb{E}_{h(g|\rho_{\mathcal{X}}(g')\mathbf{x})}[\rho_{\mathcal{Y}}(g)\phi_\theta(\rho_{\mathcal{X}}(g)^{-1}\rho_{\mathcal{X}}(g')\mathbf{x})].
    \end{align}
    Let $m = g'^{-1}g$, then we have $g = g'm$. Since $h$ is $G$ equivariant, we have $h(g|\rho_{\mathcal{X}}(g')\mathbf{x}) = h(g'm|\rho_{\mathcal{X}}(g')\mathbf{x}) = h(m|\mathbf{x})$. Therefore, \equationref{eq:fix_universal_00} becomes the following:
    \begin{align}
        \kappa_{\theta}(\rho_{\mathcal{X}}(g')\mathbf{x}) &= \mathbb{E}_{h(m|\mathbf{x})}[\rho_{\mathcal{Y}}(g'm)\phi_\theta(\rho_{\mathcal{X}}(g'm)^{-1}\rho_{\mathcal{X}}(g')\mathbf{x}] \nonumber \\ 
        &= \rho_{\mathcal{Y}}(g') \mathbb{E}_{h(m|\mathbf{x})}[\rho_{\mathcal{Y}}(m) \phi_\theta(\rho_{\mathcal{X}}(m)^{-1}\mathbf{x})] \nonumber \\ 
        &= \rho_{\mathcal{Y}}(g') \kappa_{\theta}(\mathbf{x}),
    \end{align}
    showing the $G$ equivariance of $\kappa_\theta$.

    We now prove the universality of $\kappa_\theta$.
    Assume a compact set $\mathcal{K} \subset \mathcal{X}$ is given.
    Let us denote $\mathcal{M}_{\mathcal{K}} = \cup_{\mathbf{x}\in\mathcal{K}}\,\mathrm{supp}\,h(g|\mathbf{x})$ and $\mathcal{N}_{\mathcal{K}} = \{g^{-1} | g\in \mathcal{M}_{\mathcal{K}}\}$.
    Since $h$ is compactly supported on $G$, by definition $\mathcal{M}_{\mathcal{K}}$ is compact.
    Furthermore, $\mathcal{N}_{\mathcal{K}}$ is also compact as it is image of a compact set $\mathcal{M}_{\mathcal{K}}$ under matrix inversion operator $g\mapsto g^{-1}$ which is continuous on $\mathrm{GL}(n)$.
    % Notice that $\mathcal{M}_{\mathcal{K}}$ is the support of $h(\cdot | \mathcal{K})$. \\
    Let $\psi: \mathcal{X} \to \mathcal{Y}$ be an arbitrary $G$ equivariant function. By equivariance of $\psi$, we have: 
    \begin{align}\label{eq:fix_universal_0}
        \|\psi(\mathbf{x}) - \kappa_{\theta}(\mathbf{x})\| &= \|\psi(\mathbf{x}) - \mathbb{E}_{h(g|\mathbf{x})}[\rho_{\mathcal{Y}}(g) \phi_\theta(\rho_{\mathcal{X}}(g)^{-1}\mathbf{x})]\| \nonumber \\
        &= \|\mathbb{E}_{h(g|\mathbf{x})}[\psi(\mathbf{x})] - \mathbb{E}_{h(g|\mathbf{x})}[\rho_{\mathcal{Y}}(g) \phi_\theta(\rho_{\mathcal{X}}(g)^{-1}\mathbf{x})] \| \nonumber \\
        &= \|\mathbb{E}_{h(g|\mathbf{x})}[\rho_{\mathcal{Y}}(g) \psi(\rho_{\mathcal{X}}(g)^{-1} \mathbf{x})] - \mathbb{E}_{h(g|\mathbf{x})}[\rho_{\mathcal{Y}}(g) \phi_\theta(\rho_{\mathcal{X}}(g)^{-1}\mathbf{x})] \|.
    \end{align}
    Since the union of the support $\mathcal{M}_{\mathcal{K}} = \cup_{\mathbf{x}\in\mathcal{K}}\,\mathrm{supp}\,h(g|\mathbf{x})$ is compact and $\mathcal{Y}$ is finite-dimension, there exist $c > 0$ such that $\|\rho_{\mathcal{Y}}(g)\| \leq c \forall g \in \mathcal{M}_{\mathcal{K}}$. Therefore, \equationref{eq:fix_universal_0} becomes:
    \begin{align}\label{eq:fix_universal_1}
        \|\psi(\mathbf{x}) - \kappa_{\theta}(\mathbf{x})\| &\leq \max_{g \in \mathcal{M}_{\mathcal{K}}}\|\rho_{\mathcal{Y}}(g)\|\ \mathbb{E}_{h(g| \mathbf{x})}\|\psi(\rho_{\mathcal{X}}(g)^{-1} \mathbf{x}) - \phi_\theta(\rho_{\mathcal{X}}(g)^{-1}\mathbf{x})\| \nonumber \\ 
        &\leq c \ \mathbb{E}_{h(g| \mathbf{x})}\|\psi(\rho_{\mathcal{X}}(g^{-1}) \mathbf{x}) - \phi_\theta(\rho_{\mathcal{X}}(g^{-1})\mathbf{x})\|.
    \end{align}
    Let us define the set $\mathcal{K}_{\mathrm{sym}} = \cup_{g\in\mathcal{N}_{\mathcal{K}}} \rho_{\mathcal{X}}(g)\mathcal{K}$.
    Since $\mathcal{N}_{\mathcal{K}}$ is compact and $\mathcal{X}$ is finite-dimension, the set $\{\rho_{\mathcal{X}}(g)|g\in\mathcal{N}_{\mathcal{K}} \}$ is compact, which implies the compactness of $\mathcal{K}_{\mathrm{sym}}$.
    Since $\phi_\theta$ is a universal approximator, for any $\epsilon > 0$, there exists a choice of parameters $\theta$ such that:
    \begin{align}
        \max_{g \in \mathcal{N}_{\mathcal{K}}} \|\psi(\rho_{\mathcal{X}}(g) \mathbf{x}) - \phi_\theta(\rho_{\mathcal{X}}(g)\mathbf{x})\| \leq \epsilon/c,
    \end{align}
    for all $\mathbf{x} \in \mathcal{K}$. Therefore, \equationref{eq:fix_universal_1} becomes: 
    \begin{align}
        \|\psi(\mathbf{x}) - \kappa_{\theta}(\mathbf{x})\| &\leq c \text{} \max_{g \in \mathcal{N}_{\mathcal{K}}}\|\psi(\rho_{\mathcal{X}}(g) \mathbf{x}) - \phi_\theta(\rho_{\mathcal{X}}(g)\mathbf{x})\| \leq \epsilon,
    \end{align}
    for all $\mathbf{x} \in \mathcal{K}$. This finishes the proof.
\end{proof}
So far, we have proven universality of symmetrized function $\kappa_\theta$ over $h(g|\mathbf{x})$ compactly supported on $G$ (\definitionref{defn:compact_contract}).
We now prove for symmetrized function $\Phi_{\theta,\omega}$ over parameterized distribution $p_\omega(g|\mathbf{x})$ which is approximately compactly supported on $G$ (\definitionref{defn:approximate_compact_contract}).
\begin{theorem}\label{universality}
    The symmetrized function $\Phi_{\theta, \omega}$ in \equationref{eq:symmetrization_reform} is a universal approximator of $G$ equivariant functions if $\phi_\theta$ is a continuous universal approximator and the parameterized probabilistic distribution $p_\omega$ is approximately compactly supported on $G$. 
\end{theorem}
\begin{proof}
    We prove that for arbitrary $G$ equivariant function $\psi: \mathcal{X} \to \mathcal{Y}$, for any given compact set $\mathcal{K} \subset \mathcal{X}$ and $\epsilon > 0$, there exists a choice of parameters $\theta, \omega$ such that $\|\psi(\mathbf{x}) - \Phi_{\theta, \omega}(\mathbf{x}))\| \leq \epsilon$ holds for all $\mathbf{x} \in \mathcal{K}$.
    First, given that $p_\omega$ is approximately compactly supported on $G$, from \definitionref{defn:approximate_compact_contract} we obtain the distribution $h(g|\mathbf{x})$ which is compactly supported on $G$, and a compact set $\mathcal{H}$ that includes the union of the support $\cup_{\mathbf{x}\in\mathcal{K}}\,\mathrm{supp}\,p_\omega(g|\mathbf{x})$.
    % Consider the given compact set $\mathcal{K} \subset \mathcal{X}$ and $\epsilon > 0$.
    Based on that, we define $\kappa_{\theta}: \mathcal{X} \to \mathcal{Y}$ following \lemmaref{lemma:universal_contract} as $\kappa_{\theta}(\mathbf{x}) = \mathbb{E}_{h(g|\mathbf{x})}[\rho_{\mathcal{Y}}(g) \phi_\theta(\rho_{\mathcal{X}}(g)^{-1}\mathbf{x})]$.
    With the triangle inequality of the metric induced by $\|\cdot\|$, we have:
    \begin{align}
        \|\psi(\mathbf{x}) - \Phi_{\theta, \omega}(\mathbf{x}))\| \leq \|\psi(\mathbf{x}) - \kappa_{\theta}(\mathbf{x})\| + \|\kappa_{\theta}(\mathbf{x}) - \Phi_{\theta, \omega}(\mathbf{x}))\|.
    \end{align}
    According to \lemmaref{lemma:universal_contract}, $\kappa_{\theta}$ is $G$ equivariant and there exists a choice of parameter $\theta_*$ such that $\|\psi(\mathbf{x}) - \kappa_{\theta_*}(\mathbf{x})\| \leq \epsilon/2, \text{ } \forall \mathbf{x} \in \mathcal{K}$. So we only need to prove there exists a choice of parameter $\omega$ such that $\|\kappa_{\theta_*}(\mathbf{x}) - \phi_{\theta_*, \omega}(\mathbf{x}))\| \leq \epsilon/2, \text{ } \forall \mathbf{x} \in \mathcal{K}$.
    % &= \| \int_{G}[\rho_{\mathcal{Y}}(g) f_{\theta_*}(\rho_{\mathcal{X}}(g)^{-1}\mathbf{x})] h(g|\mathbf{x}) d_g - \int_{\mathrm{GL}(n)}[\rho^{gl}_{\mathcal{Y}}(g) f_{\theta_*}(\rho^{gl}_{\mathcal{X}}(g)^{-1} \mathbf{x})] p_\omega(g|\mathbf{x})d_g \|
    We first write: 
    \begin{align}\label{eq:universal_0}
        \|\kappa_{\theta_*}(\mathbf{x}) - \Phi_{\theta_*, \omega}(\mathbf{x}))\| &= \|\mathbb{E}_{h(g|\mathbf{x})}[\rho_{\mathcal{Y}}(g) f_{\theta_*}(\rho_{\mathcal{X}}(g)^{-1}\mathbf{x})] - \mathbb{E}_{p_\omega(g|\mathbf{x})}[\rho^{gl}_{\mathcal{Y}}(g) f_{\theta_*}(\rho^{gl}_{\mathcal{X}}(g)^{-1} \mathbf{x})]\|.
    \end{align}
    Since $p_\omega$ is approximately compactly supported on $G$, for any $0 < \alpha < 1$, there exists a choice of parameter $\omega_*$ such that (1) $p_{\omega_*}(g | \mathbf{x}) \geq (1 - \alpha)h(g|\mathbf{x})$ holds for all $g \in G, \mathbf{x} \in \mathcal{K}$, and (2) the union of the support $\cup_{\mathbf{x}\in\mathcal{K}}\,\mathrm{supp}\,p_{\omega_*}(g|\mathbf{x})$ is a subset of the compact set $\mathcal{H}$.
    Given that, we define an unnormalized distribution $p'_{\omega_*}(g|\mathbf{x})$ as follows:
    \begin{equation}\label{eq:universal_2}
        \begin{cases}
            p'_{\omega_*}(g | \mathbf{x}) = p_{\omega_*}(g | \mathbf{x}) - (1 - \alpha)h(g | \mathbf{x})\quad \forall g \in G, \\ 
        p'_{\omega_*}(g | \mathbf{x}) = p_{\omega_*}(g | \mathbf{x})\quad\forall g \notin G. \\
        \end{cases}
    \end{equation}
    % for all $\mathbf{x} \in \mathcal{K}$.\\
    By setting $\omega = \omega_*$, the right side of \equationref{eq:universal_0} becomes: 
    \begin{align}
        \|&\kappa_{\theta_*}(\mathbf{x}) - \Phi_{\theta_*, \omega}(\mathbf{x}))\|\nonumber\\
        &= \left\| \int_{G}[\rho_{\mathcal{Y}}(g) f_{\theta_*}(\rho_{\mathcal{X}}(g)^{-1}\mathbf{x})] h(g|\mathbf{x}) d_g - \int_{\mathrm{GL}(n)}[\rho^{gl}_{\mathcal{Y}}(g) f_{\theta_*}(\rho^{gl}_{\mathcal{X}}(g)^{-1} \mathbf{x})] p_{\omega_*}(g|\mathbf{x})d_g \right\| \nonumber\\
        &= \left\|\alpha \int_{G}[\rho_{\mathcal{Y}}(g) f_{\theta_*}(\rho_{\mathcal{X}}(g)^{-1}\mathbf{x})] h(g|\mathbf{x}) d_g -  \int_{\mathrm{GL}(n)}[\rho^{gl}_{\mathcal{Y}}(g) f_{\theta_*}(\rho^{gl}_{\mathcal{X}}(g)^{-1} \mathbf{x})] p'_{\omega_*}(g|\mathbf{x})d_g \right\|
    \end{align}
    By denoting $A(\mathbf{x}, g) = \rho_{\mathcal{Y}}(g) f_{\theta_*}(\rho_{\mathcal{X}}(g)^{-1}\mathbf{x}) $ and $B(\mathbf{x}, g) = \rho^{gl}_{\mathcal{Y}}(g) f_{\theta_*}(\rho^{gl}_{\mathcal{X}}(g)^{-1} \mathbf{x})$, we have:
    \begin{align}\label{eq:universal_3}
        \|\kappa_{\theta_*}(\mathbf{x}) - \Phi_{\theta_*, \omega}(\mathbf{x}))\| &= \left\|\alpha \int_{G}A(\mathbf{x},g)h(g|\mathbf{x}) d_g - \int_{\mathrm{GL}(n)} B(\mathbf{x}, g)p'_{\omega_*}(g|\mathbf{x})d_g\right\| \nonumber \\ 
        &\leq \alpha \int_{G}\|A(\mathbf{x}, g)\|h(g|\mathbf{x}) d_g +  \int_{\mathrm{GL}(n)}\|B(\mathbf{x}, g)\|p'_{\omega_*}(g|\mathbf{x})d_g.
    \end{align}
    % \begin{align}\label{eq:universal_3}
    %     \|\kappa_{\theta_*}(\mathbf{x}) - \phi_{\theta_*, \omega_*}(\mathbf{x}))\| &= \|\alpha \int_{G}[\rho_{\mathcal{Y}}(g) f_{\theta_*}(\rho_{\mathcal{X}}(g)^{-1} \mathbf{x})] h(g|\mathbf{x})d_g - \int_{\mathrm{GL}(n)}[\rho^{gl}_{\mathcal{Y}}(g) f_{\theta_*}(\rho^{gl}_{\mathcal{X}}(g)^{-1} \mathbf{x})] p'_{\omega_*}(g|\mathbf{x})d_g\| \nonumber \\
    %     &\leq \alpha \|\int_{G}[\rho_{\mathcal{Y}}(g) f_{\theta_*}(\rho_{\mathcal{X}}(g)^{-1} \mathbf{x})] h(g|\mathbf{x})d_g\| + \|\int_{\mathrm{GL}(n)}[\rho^{gl}_{\mathcal{Y}}(g) f_{\theta_*}(\rho^{gl}_{\mathcal{X}}(g)^{-1} \mathbf{x})] p'_{\omega_*}(g|\mathbf{x})d_g\| \nonumber \\
    %     &\leq \alpha \int_{G}\|\rho_{\mathcal{Y}}(g) f_{\theta_*}(\rho_{\mathcal{X}}(g)^{-1} \mathbf{x})\| h(g|\mathbf{x})d_g + \int_{\mathrm{GL}(n)}\|\rho^{gl}_{\mathcal{Y}}(g) f_{\theta_*}(\rho^{gl}_{\mathcal{X}}(g)^{-1} \mathbf{x})\| p'_{\omega_*}(g|\mathbf{x})d_g \nonumber \\
    % \end{align}
    Since $f_{\theta_*}$ is continuous and the set $\mathcal{M}_{\mathcal{K}} = \cup_{\mathbf{x}\in\mathcal{K}}\,\mathrm{supp}\,h(g|\mathbf{x})$ is compact, we have that the set $\cup_{\mathbf{x}\in\mathcal{K}}\{A(\mathbf{x}, g):g\in \mathrm{supp}\,h(g|\mathbf{x})\}$ is compact.
    Thus, there exists $C_1 > 0$ such that $\|A(\mathbf{x}, g)\| \leq C_1$ for all $\mathbf{x} \in \mathcal{K}, g \in \mathcal{M}_{\mathcal{K}}$.
    % Let us define the set $\mathcal{N}_{\mathcal{K}} = \{g \in \mathrm{GL}(n) | \text{ } p'_{\omega_*}(g | \mathbf{x}) > 0 \text{ for some } \mathbf{x} \in \mathcal{K} \}$.
    Let us define the set $\mathcal{N}_{\mathcal{K}} = \cup_{\mathbf{x}\in\mathcal{K}}\,\mathrm{supp}\,p'_{\omega_*}(g|\mathbf{x})$ and assume that $\mathcal{N}_\mathcal{K}\subseteq\mathrm{GL}(n)$.
    From the property of $p_{\omega_*}$ that the union of the support is bounded by the compact set $\mathcal{H}$, and the definition of $p'_{\omega_*}$ in \equationref{eq:universal_2},
    % According to the second property of $p_{\omega_*}$ in \equationref{eq:universal_1} and definition of $p'_{\omega_*}$ in \equationref{eq:universal_2}, 
    we can see that $\mathcal{N}_{\mathcal{K}}$ is bounded by the compact set $\mathcal{H}$.
    This also leads to the compactness of $\cup_{\mathbf{x}\in\mathcal{K}}\{B(\mathbf{x},g):g\in\mathrm{supp}\,p'_{\omega_*}(g|\mathbf{x})\}$.
    Therefore, there exists $C_2 > 0$ such that $\|\rho^{gl}_{\mathcal{Y}}(g) f_{\theta_*}(\rho^{gl}_{\mathcal{X}}(g)^{-1} \mathbf{x})\| \leq C_2$ for all $\mathbf{x} \in \mathcal{K}$ and $g \in \mathcal{N}_{\mathcal{K}}$.
    As a result, \equationref{eq:universal_3} becomes: 
    \begin{align}
        \|\kappa_{\theta_*}(\mathbf{x}) - \Phi_{\theta_*, \omega_*}(\mathbf{x}))\| &\leq \alpha \int_{G}C_1 h(g|\mathbf{x})d_g + \int_{\mathrm{GL}(n)}C_2  p'_{\omega_*}(g|\mathbf{x})d_g \nonumber \\ 
        &= \alpha C_1 + \alpha C_2.
    \end{align}
    Notice that $C_1, C_2$ are dependent on distribution $h(g|\mathbf{x})$ and the set $\mathcal{H}$, but not $\omega$. Therefore, by choosing $\alpha = \frac{\epsilon}{2(C_1 + C_2)}$, and setting $\omega = \omega_*$ accordingly, we have: 
    \begin{align}
        \|\kappa_{\theta_*}(\mathbf{x}) - \phi_{\theta_*, \omega_*}(\mathbf{x}))\| \leq \epsilon/2, \text{ } \forall \mathbf{x} \in \mathcal{K}
    \end{align}
    This finishes the proof.
\end{proof}

\subsection{Experimental Details}\label{sec:experimental_details}
\subsubsection{Rotated MNIST}

In this section, we supplement the implementation and training details for the Rotated MNIST experiment. 
We employ exactly the same CNN architecture as \cite{kaba2023equivariance} to implement our base function $\phi_\theta$, which has 7 layers with hidden dimensions of 32, 64, 128 for layers 1 -- 3, layers 4 -- 6, and layer 7 respectively.
At layers 4 and 7, a $5 \times 5$ convolution with stride 2 are used instead of pooling.
Other convolutions use $3 \times 3$ filters with stride 1.
Each convolution is followed by batch normalization and ReLU activation.
Dropout of $p=0.4$ is used at layers 4 and 7.

For our $\mathrm{SO}(2)$ equivariant symmetrizer $q_\omega(\mathbf{x}, \boldsymbol{\epsilon})$, given an input image $\mathbf{x} \in \mathbb{R}^{28 \times 28 \times 1}$, we preprocess it into a tensor format and apply an EMLP \citep{finzi2021practical} on it.
In more detail, firstly, we construct a coordinate map $\mathbf{C} \in \mathbb{R}^{28 \times 28 \times 2}$ where the central pixel has coordinate $(0, 0)$ and each corner has coordinate $(\pm 14, \pm 14)$.
This coordinate map will be shared for all images.
We construct a tensor $\mathbf{v} \in \mathbb{R}^{28 \times 28 \times 3}$ by concatenating $\mathbf{x}$ and $\mathbf{C}$ across the channel dimension ($\mathbf{v}[:,:,0] = \mathbf{x}, \mathbf{v}[:,:,1:] = \mathbf{C}$).
Then, we flat the first two dimension to have tensor $\mathbf{v} \in \mathbb{R}^{784 \times 3}$.
Now, each row of $\mathbf{v}$ corresponds to a pixel in the original image, where the first channel is the pixel value, while the last two channels are pixel's coordinates.
Next, we sort rows in $\mathbf{v}$ in an ascending order of the first column ($\mathbf{v}[:,0]$) and mask out all rows with a pixel value less than a predefined threshold $t=0.2$.
Then, we keep the top $m=200$ rows and achieve a new tensor $\mathbf{v} \in \mathbb{R}^{m \times 3}$.
Lastly, we extract two tensors $\mathbf{v}_1, \mathbf{v}_2$ from $\mathbf{v}$ as following: $\mathbf{v}_1 = \mathbf{v}[:,0]^\top \in \mathbb{R}^{1 \times m}, \mathbf{v}_2 = \mathbf{v}[:,1:]^\top \in \mathbb{R}^{2 \times m}$.
Notice that these processing steps do not break rotation equivariance and $\mathbf{v}_1$ is $\mathrm{SO}(2)$ invariant while $\mathbf{v}_2$ is $\mathrm{SO}(2)$ equivariant with respect to input image $\mathbf{x}$.

To train symmetrization models, we employ faithful noise $\boldsymbol{\epsilon} \in \mathbb{R}^{2 \times d_{\mathcal{E}}}$ with $d_\mathcal{E}$ set to 10. To implement the symmetrizer, we leverage a 2-layers EMLP with hidden dimension 64 that has the following input and output representations in notation of \cite{finzi2021a}:
\begin{equation}
    \begin{cases}
        \text{repin} = (d_{\mathcal{E}} + m)T_{(0)} + mT_{(1)} \\ 
        \text{repout} = 2T_{(1)}
    \end{cases}
\end{equation}
With the symmetrizer $q_\omega(\mathbf{x}, \boldsymbol{\epsilon})$ designed in this manner, $q_\omega$ is guaranteed to always output $2 \times 2$ matrices in $\mathrm{SO}(2)$ equivariant manner with respect to the input image $\mathbf{x}$.

When we compute the input transformation $q_\omega(\mathbf{x}, \boldsymbol{\epsilon})^{-1}\cdot\mathbf{x}$ as in \equationref{eq:final_model}, we find that since $q_\omega(\mathbf{x}, \boldsymbol{\epsilon})$ is a neural feature, directly applying matrix inverse to it harms the stability of training.
To enhance stability, we assume as if $q_\omega(\mathbf{x}, \boldsymbol{\epsilon})$ is already close to a $\mathrm{SO}(2)$ matrix, and employ the property of $\mathrm{SO}(n)$ matrices $\rho(g)^{-1}=\rho(g)^\top$ to compute an approximation of inverse.
This approximation becomes more exact as orbit distance training progresses $q_\omega(\mathbf{x}, \boldsymbol{\epsilon}) \to \rho(g)\in\mathrm{SO}(2)$.
Indeed, we observe this significantly improves training stability while not harming the quality of learned transformations $q_\omega(\mathbf{x}, \boldsymbol{\epsilon})$.
% To enhance stability of training, we employ the transpose proxy to approximate the exact inverse operation on $q_\omega(\mathbf{x}, \boldsymbol{\epsilon})$. It is known for $SO(2)$ (or $SO(n)$ in general) that the proxy $Q^T$ is exactly $Q^{-1}$ if $Q \in SO(2)$. Our intuition is that this proxy approximates the inverse operation perfectly as $q_\omega(\mathbf{x}, \boldsymbol{\epsilon}) \to SO(2)$ while improving training stability significantly.
To apply the transformation $q_\omega(\mathbf{x}, \boldsymbol{\epsilon})^{-1}$ to input image $\mathbf{x}$, we follow implementation of \cite{kaba2023equivariance} and employ computer vision library Kornia \citep{riba2020kornia}.

\begin{figure}[!t]
\floatconts
  {fig:rotated_mnist}
  {\caption{Example transformed images for Rotated MNIST dataset.
  The first row include images from the original dataset, while the second row are those images transformed by output of our learned symmetrizer $q_\omega(\mathbf{x}, \boldsymbol{\epsilon})$.
  It is clear from these figures that transformations associated with $q_\omega(\mathbf{x}, \boldsymbol{\epsilon})$ is purely rotation.}}
  {\includegraphics[width=0.8\linewidth]{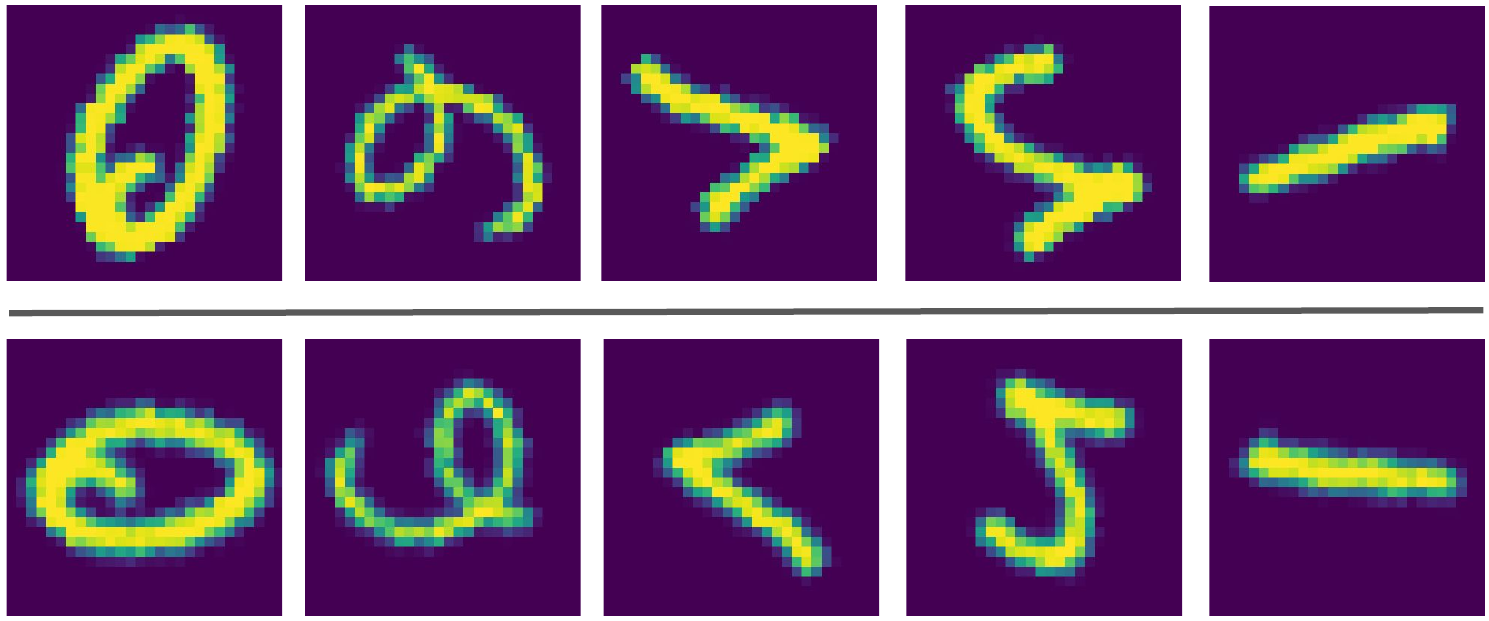} \\[-1.5em]}
  {\vspace{-0.2cm}}
\end{figure}

We train the models by minimizing the cross entropy loss on classification task jointly with L1 norm for orbit distance minimization using $\lambda=1.0$.
The models are trained for 1000 epochs using Adam optimizer and a learning rate of 0.0003.
All symmetrization methods in Table \ref{tab:rotated_mnist} are trained with this same setting for fair comparison.
Some example transformations $q_\omega(\mathbf{x}, \boldsymbol{\epsilon})^{-1}\cdot \mathbf{x}$ learned by our models trained with orbit distance minimization is shown in \figureref{fig:rotated_mnist}, which verifies that valid group representations are indeed learned.

\subsubsection{Particle Scattering}

In this section, we supplement the implementation and training details for the Particle Scattering experiment.
We use 3-layer MLP with 128 hidden dimensions and SiLU~\citep{hendrycks2016gaussian} activation function as our base function $\phi_\theta$.

Given an input $\mathbf{x}\in\mathbb{R}^{4\times 4}$, for the $\mathrm{O}(1, 3)$ equivariant symmetrizer $q_\omega(\mathbf{x}, \boldsymbol{\epsilon})$, we use a 3-layer Scalar MLP~\citep{villar2021scalars} with 28 hidden dimensions and SiLU activation, preceded by an $\mathrm{O}(1, 3)$ equivariant featurization procedure.
In more detail, we first sample the noise variable $\boldsymbol{\epsilon}\in\mathbb{R}^{4\times d_\mathcal{E}}$ from a compactly supported distribution $p(\boldsymbol{\epsilon})$ which is invariant to $\mathrm{O}(1, 3)$ under trivial representation $\rho_\mathcal{E}(g) = \mathbf{I}$.
We use elementwise uniform distribution $\boldsymbol{\epsilon}_{ij}\sim\mathrm{Unif}[\mathbf{a}_{ij}, \mathbf{a}_{ij} + \mathbf{b}_{ij}]$ with trainable offset $\mathbf{a}\in\mathbb{R}^{4\times d_\mathcal{E}}$ and scale $\mathbf{b}\in\mathbb{R}^{4\times d_\mathcal{E}}$ initialized as $\mathbf{1}$ and $\mathbf{0}$ respectively.
Then, before Scalar MLP, we transform $\boldsymbol{\epsilon}$ into a feature matrix $\mathbf{z}\in\mathbb{R}^{4\times d_\mathcal{E}}$ with a simple procedure equivariant to $\mathrm{O}(1, 3)$ transformations of the input $\mathbf{x}$.
We interpret the sampled noise $\boldsymbol{\epsilon}$ as the the Minkowsky inner product between input $\mathbf{x}$ and feature $\mathbf{z}$ (which is unknown at this point) $\boldsymbol{\epsilon} = \mathbf{x}^\top\boldsymbol{\Lambda}\mathbf{z}$ where $\boldsymbol{\Lambda}=\mathrm{diag}([+1, -1, -1, -1])$, from which we obtain $\mathbf{z} = (\mathbf{x}^\top\boldsymbol{\Lambda})^{-1}\boldsymbol{\epsilon}$.
As Minkowsky inner product $\boldsymbol{\epsilon}$, or space-time interval, is $\mathrm{O}(1, 3)$ invariant as can be seen in $\mathbf{x}^\top\boldsymbol{\Lambda}\mathbf{z} = (g\cdot \mathbf{x})^\top\boldsymbol{\Lambda}(g\cdot\mathbf{z}) =\mathbf{x}^\top\rho(g)^\top\boldsymbol{\Lambda}\rho(g)\mathbf{z} = \mathbf{x}^\top\boldsymbol{\Lambda}\mathbf{z}$, having fixed the noise $\boldsymbol{\epsilon}$, transforming $\mathbf{x}\mapsto g\cdot\mathbf{x}$ transforms $\mathbf{z}\mapsto g\cdot\mathbf{z}$ accordingly.
For Canonicalization~\citep{kaba2023equivariance} that has to drop stochasticity, we simply use deterministic $\boldsymbol{\epsilon}=\mathbf{a}$ to obtain the feature $\mathbf{z}$.
Then, we then use this feature $\mathbf{z}$ to supplement the input $\mathbf{x}$ by addition or channel concatenation, and provide the combined feature as an input to Scalar MLP and obtain the output $\mathbf{h}\in\mathbb{R}^{4\times 4}$.
In our experiments, we find that this featurization significantly and consistently improves orbit distance training.
Note that our framework and theoretical results are not altered, as the featurization $(\mathbf{x}, \boldsymbol{\epsilon})\mapsto\mathbf{z}$ is $\mathrm{O}(1, 3)$ equivariant and can be interpreted as a part the equivariant symmetrizer $q_\omega:(\mathbf{x}, \boldsymbol{\epsilon})\mapsto\mathbf{h}$.

Similar as in the $\mathrm{SO}(2)$ experiment, to avoid computing inverse of neural feature during the input transformation $q_\omega(\mathbf{x}, \boldsymbol{\epsilon})^{-1}\cdot\mathbf{x}$ (\equationref{eq:final_model}), we assume as if $q_\omega(\mathbf{x}, \boldsymbol{\epsilon})$ is already close to a $\mathrm{O}(1, 3)$ matrix, and employ the property of $\mathrm{O}(1, 3)$ matrices $\rho(g)^{-1}=\boldsymbol{\Lambda}\rho(g)^\top\boldsymbol{\Lambda}$ to compute an approximation of inverse.
This approximation becomes more exact as orbit distance training progresses $q_\omega(\mathbf{x}, \boldsymbol{\epsilon}) \to \rho(g)\in\mathrm{O}(1, 3)$, and we observe this significantly improves training stability while not harming the quality of learned transformations $q_\omega(\mathbf{x}, \boldsymbol{\epsilon})$.

We train the models by minimizing the mean squared error loss for regression task jointly with L1 norm on the orbit distance minimization using $\lambda=1$.
All models are trained for 1,000 epochs with batch size 1,000 using Adam optimizer with a learning rate of 0.003.
All methods in \tableref{tab:particle_scattering} are trained with this same setting.
Both of our models $q_\omega(\mathbf{x}, \boldsymbol{\epsilon})\mapsto\mathbf{h}$ consistently achieve orbit distance $\|f(\mathbf{h}) - f(\mathbf{I})\|$ of around $0.02$--$0.03$ on unseen inputs $\mathbf{x}$, while random Gaussian matrices have loss of around $\approx50$.
This supports that valid group representations are learned by our approach.

\end{document}